\documentclass[letterpaper, 10 pt, conference]{ieeeconf}  

\IEEEoverridecommandlockouts                              
\usepackage{amsmath,amssymb}
\usepackage{xcolor}
\usepackage{optidef}
\usepackage[normalem]{ulem}
\usepackage{lipsum}
\usepackage{mathtools}
\usepackage{cuted}
\usepackage{algorithm}
\usepackage{algpseudocode}
\usepackage{breqn}
\usepackage{graphicx}
\usepackage{subfig}
\usepackage{psfrag}
\usepackage[bb=dsserif]{mathalpha}
\usepackage{cite}

\newtheorem{theorem}{\hspace{0pt}\bf Theorem}

\newtheorem{coro}{\hspace{0pt}\bf Corollary}
\newtheorem{prop}{\hspace{0pt}\bf Proposition}
\newcommand\redsout{\bgroup\markoverwith{\textcolor{red}{\rule[0.5ex]{2pt}{0.4pt}}}\ULon}

\definecolor{mygreen}{rgb}{0.10,0.50,0.10}

\title{\LARGE \bf
A Bi-Level Optimization Method for Redundant Dual-Arm Minimum Time Problems 
}

\author{Jonathan Fried and Santiago Paternain
\thanks{The authors are with the Department of Electrical, Computer, and Systems Engineering, Rensselaer Polytechnic Institute. Email: \{friedj2, paters\}@rpi.edu }
}

\begin{document}

\maketitle

\begin{abstract}

In this work, we present a method for minimizing the time required for a redundant dual-arm robot to follow a desired relative Cartesian path at constant path speed by optimizing its joint trajectories, subject to position, velocity, and acceleration limits. The problem is reformulated as a bi-level optimization whose  lower level is a convex, closed-form subproblem that maximizes path speed for a fixed trajectory, while the upper level updates the trajectory using a single-chain kinematic formulation and the subgradient of the lower-level value. Numerical results demonstrate the effectiveness of the proposed approach.

\end{abstract}

\section{Introduction}

In several industrial tasks performed by robots, the Cartesian path to be followed by the end-effector requires fewer degrees of freedom than what the manipulator possesses. In many applications \cite{yundou2018,mehmet2019,owen2022,wen2020,heping2008}, rotation around the attached tool center point (TCP) is considered a free axis, resulting in a kinematically redundant system \cite{conkur_1997,leger2016}. This allows adjustment of the robot pose in ways that optimize multiple metrics. 
In the context of our work, we attempt to minimize the completion time under constant path speed. Inspired by applications in additive manufacturing such as spray coating, we consider constant path velocity. The latter is a proxy for material deposition, which must be uniform to ensure the quality of the process~\cite{xiong2020,coutinho2023,Li2018,DING2023683,Tan2017,SERAJ201924}.

To improve performance, manufacturers are turning to dual-arm systems to overcome single-robot limitations. \cite{jiang2022} addresses this from a control perspective, combining adaptive control with radial basis functions to handle dynamic uncertainties, but does not exploit system redundancy. \cite{lee2015} proposes an optimization approach inspired by human bimanual asymmetry, assigning fine motion and force tasks to the right arm and coarse, supportive roles to the left. \cite{chaki2024} approaches the problem via motion planning, using inverse kinematics and quadratic programming with inequality constraints on relative position to predefine allowable error and achieve locally optimal, precise motions.

To minimize time in dual-arm assembly tasks, \cite{YING2021} combines Bi-directional RRT with an LSTM network, using a greedy heuristic to enhance sampling. It focuses on trajectory generation rather than redundancy exploitation, making it unsuitable for tasks with tight error margins. \cite{kurosu2017} uses mixed integer linear programming with RRT to optimize operation time in dual-arm pick-and-place. MILP assigns object handling and order per arm, while RRT ensures collision-free paths. This speeds up planning compared to pure RRT and balances arm motions, but lacks the precision and coordination needed for tasks such as handovers, nor explores joint redundancies. \cite{DILEVA2024} addresses dual-arm manipulation as a time-optimal control problem, yet since it optimizes for both Cartesian trajectory and force, redundancy usage is limited.

We address dual-arm motion planning by minimizing path traversal time under strict accuracy constraints \cite{he2023}. Our algorithm efficiently finds locally optimal solutions by leveraging the structure of the joint motion problem, while avoiding recursive usage of inverse kinematics. We use a bi-level formulation with a convex, efficiently solvable lower-level subproblem, similar to \cite{vers2009,tang2019}. This solution can be used to find descent directions with respect to the high-level problem. Unlike \cite{tang2019}, which assumes differentiability of the resulting upper-level problem, we show that descent directions can still be efficiently computed even when that assumption does not hold. By an appropriate selection of the reference frame, we describe the relative path between arms independently of joint positions, thus simplifying the Cartesian error computation. Furthermore, the usage of the relative Jacobian~\cite{roberts2015} enables the application of techniques from \cite{friedarxiv} to solve the upper-level problem. Our curve parameterization updates the entire joint trajectory at each iteration, unlike point-wise optimizations in quadratic programming \cite{he2023,chaki2024} or control-based methods \cite{jiang2022}.

This work is organized as follows. In Section \ref{sec_problem_formulation} we formulate the redundant dual-arm minimum time constant path speed problem. Section \ref{sec_bilevel} exploits its bi-level structure and provides a closed-form solution for the lower-level problem which depends only on the joint trajectories (Theorem \ref{thm:inner_cte}). This result extends our prior work on single-arm redundancy resolution~\cite{fried2024}. This solution as well as a relative formulation between the forward and differential kinematics of the two arms (Section \ref{sec:dualarmsol}) benefit the calculation of the upper-level subdifferentials, which improve the efficiency of the algorithm. This is supported by the numerical results in Section \ref{sec_numerical}. In particular, we compare the proposed approach with a general purpose nonlinear solver in an experiment inspired by a cold spraying application. We present as well the benefits of considering the dual-arm compared to the single arm solution proposed in our previous work~\cite{friedarxiv}. This is followed by concluding remarks in Section \ref{sec_conclusion}.

\section{Dual Arm Minimum Time Problem}\label{sec_problem_formulation}

Consider the problem of optimizing the joint motion of two redundant robot manipulators. Let $n_A$ and $n_B$ be the number of joints of each robot. The task is described by the relative Cartesian path, position and orientation, between the two TCPs $\chi_d \in \mathbb{R}^m$. The robotic system (shown in Figure \ref{fig:transforms}) is redundant with respect to the task, so $n = n_A + n_B > m$. 
Denote by $\mathcal{F}_{bA}$, $\mathcal{F}_{tA}$ the Manipulator A base frame and TCP frame, respectively, and $\mathcal{F}_{bB}$, $\mathcal{F}_{tB}$ those of Manipulator B.

The $\ell$-th manipulator's forward kinematics $k_\ell:\mathbb{R}^{n_\ell}\to\mathbb{R}^m$ maps joint space to Cartesian space, i.e.,
\begin{equation}
    \label{eq:fwd_kin_i} 
    \chi_\ell= k_\ell(q_\ell), \quad \mbox{for $\ell:={A,B}$},
\end{equation}
where $\chi_\ell$ is the pose of robot $i\ell$ in the base frame $\mathcal{F}_{b\ell}$. The manipulator's differential kinematics, which map joint velocity to Cartesian velocity, are given by
%
\begin{equation}
    \label{eq:dif_kin_i}
    \dot{\chi}_\ell = J_\ell(q_\ell)\dot{q}_\ell,\quad \mbox{for $\ell:={A,B}$},
\end{equation}

Finally, as we are interested in the relative Cartesian path between the two TCPs, the effective path between the TCPs of both robots can be defined as:
\begin{equation}
    \label{eq:relcartesianpath}
    (\chi)_{\mathcal{F}} = (\chi_A)_{\mathcal{F}} - (\chi_b)_{\mathcal{F}} = (k_A(q_a))_{\mathcal{F}} - (k_b(q_b))_{\mathcal{F}},
\end{equation}
where $\mathcal{F}$ is some frame of reference. In Section \ref{sec:dualarmsol} we show that there are benefits for describing the relative on the frame of reference of the tooltip of the second arm $\mathcal{F}_{tB}$.

Now, let $s_i$, where $s_0=0$ and $s_N=1$ with $i=0,\ldots,N$ denote the path length and 
$(\chi^d_i)_{\mathcal{F}} \in \mathbb{R}^m$ with $i=0,\ldots N$ the desired relative Cartesian path that the robots must trace. 

Since the dual arm robot system is redundant, multiple trajectories in joint space result in the same relative Cartesian path. Among these, we are interested in finding one with the minimum traversal time and constant path velocity.

The minimum time joint optimization problem  is equivalent to choosing $q$ for each manipulator that maximizes constant path velocity $\dot{s}$ or to minimize $t_f$, the travel time of the dual arm robot system across the path $\chi^d$  
\begin{equation}
\label{eq:objective}
t_f = \frac{1}{\dot{s}}.
\end{equation}
%
%

Each manipulator's joint $j=1,\ldots,n_\ell$ operates under position, velocity, and acceleration constraints, given by 
\begin{equation}
    \underline{q}_{\ell j} \leq q_{\ell ij} \leq \overline{q}_{\ell j},\;\dot{\underline{q}}_{\ell j} \leq \dot{q}_{\ell ij} \leq \dot{\overline{q}}_{\ell j}\;\mbox{and}\;    \ddot{\underline{q}}_{\ell j} \leq \ddot{q}_{\ell ij} \leq \ddot{\overline{q}}_{\ell j},
    \label{eq:joint_restrictions}
\end{equation}
where $q_{\ell ij}$ is the $\ell$-th manipulator's joint $j$-th position at path point $s_i$, for $\ell=\{A,B\}$, $\forall \, j =\{1,\ldots,n_\ell\}$ and $\forall \,i =\{0,\ldots,N\}$, and $\dot{\underline{q}}_{\ell j}, \ddot{\underline{q}}_{\ell j} < 0$, $\dot{\overline{q}}_{\ell j}, \ddot{\overline{q}}_{\ell j} > 0$. This assumption implies that revolute joints can rotate clockwise or counterclockwise, or move forward and backward for prismatic joints.  We approximate the joints' trajectory $q_{\ell ij}$ linearly with respect to vectors of parameters $\theta_{\ell j} \in \mathbb{R}^d$ for all $j=1,\ldots{n_\ell}$ so that 
\begin{equation}
    q_{\ell ij} = p(s_i)\theta_{\ell j},
    \label{eq:paramterization}
\end{equation}
where $p(s_i) \in \mathbb{R}^{1\times d}$ is a twice differentiable vector basis in $s$. This assumption is common in the literature (see e.g., \cite{koubiaspline,zhu2022,embry2018}). Note that each robot joint could be parameterized by a different basis, however, we use the same one for simplicity. Applying the chain rule, the joint velocities yield
\begin{equation}\label{q_dot}
\dot{q}_{\ell ij} = q_{\ell ij}^{\prime}(s_i)\dot{s}_i = p^{\prime}(s_i)\theta_{\ell j} \dot{s}_i.
\end{equation}
\begin{figure}
    \centering
    \includegraphics[width=0.9\linewidth]{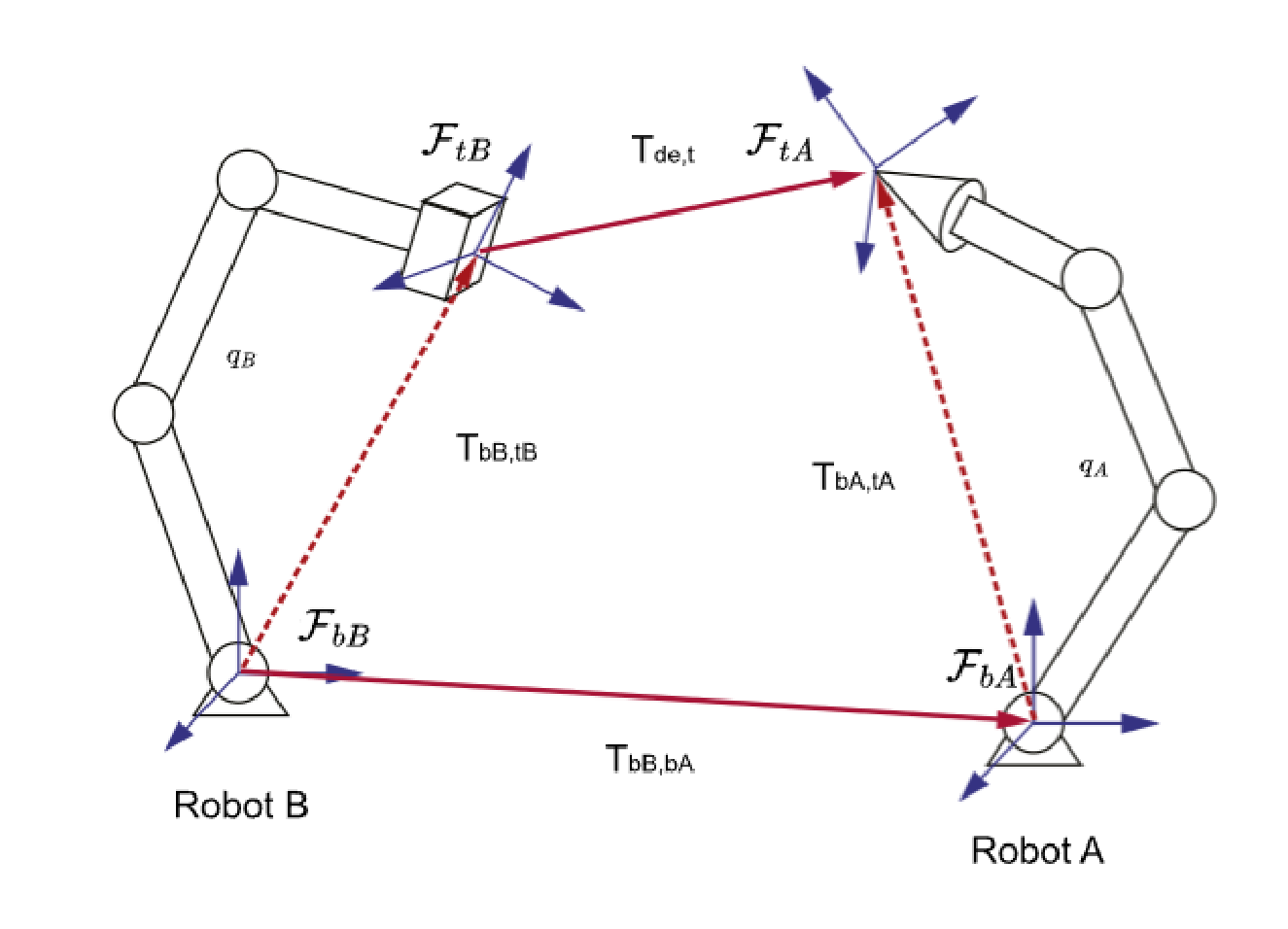}
    \caption{Setup example consisting of two arms with 6 degrees of freedom. Transformations between the frames of the dual arm robotic system \cite{roberts2015}.}
    \label{fig:transforms}
\end{figure}

Likewise, their accelerations are given by
\begin{equation}
\label{eq:sec_der}
\ddot{q}_{\ell j} = q_{\ell ij}^{\prime\prime}\dot{s}_i^2+q_{\ell ij}^{\prime}\ddot{s}_i = p^{\prime\prime}(s_i)\theta_{\ell j} \dot{s}_i^2+p^{\prime}(s_i)\theta_{\ell j}\ddot{s}_i.
\end{equation}
As we are interested in the case with constant path velocity, i.e. $\ddot{s}_i = 0$, the previous equations reduce to 
\begin{align}    
\label{eq:sec_der_2}
\dot{q}_{\ell ij} &= p^{\prime}(s_i)\theta_{\ell j} \dot{s} \\
\ddot{q}_{\ell j} &= p^{\prime\prime}(s_i)\theta_{\ell j} \dot{s}^2.
\end{align}
We define for simplicity the following notation $p_i\!\!=\!\!p(\!s_i\!), p^\prime_i\!\!=\!\!p^\prime\!(\!s_i\!), p^{\prime\prime}_i\! =\!p^{\prime\prime}\!(\!s_i\!)$, which we will use for the remainder of this work. 
Note that, since $p^\prime$ and $p^{\prime\prime}$ are not constant, the joint velocities and accelerations are not constant despite $\dot{s}$ being so. Thus, the joint limits must still be considered.


Finally, define $p-$norm error between the desired Cartesian path and the manipulators relative Cartesian path as
\begin{equation}
    \label{eq:error}
    E(\theta_A,\theta_B) = \left(\sum_{i=0}^N \left\|(\chi^d_i(\theta_A,\theta_B))_{\mathcal{F}}- (\chi_i(\theta_A,\theta_B))_{\mathcal{F}}\right\|_2^p\right)^{\frac{1}{p}},\end{equation} 
where we have written the Cartesian path as a function of the joint parameters with a slight abuse of notation. Note that in many industrial applications (\cite{he2023},\cite{coutinho2023}) this relative path $\chi_d$ is defined by a part attached to the TCP of one of the robots. In this case, describing this path in a general frame $\mathcal{F}$ may be non-trivial and depend on the robot pose.

 With this, we are now able to formalize the minimum time joint optimization problem:

\begin{mini}|s|
  {\stackrel{\dot{s}\geq 0,}{\stackrel{\theta_A \in \mathbb{R}^{n_Ad},\theta_B \in \mathbb{R}^{n_Bd}}{{}}}
  }
  { \frac{1}{\dot{s}}}{\label{opt:generaldiscrete}}{{t_f}^\star=}
   \addConstraint{}{}{E(\theta_A,\theta_B)\leq \epsilon}
 \addConstraint{}{}{\underline{q}_{\ell j} \leq p_i\theta_{\ell j}\leq \overline{q}_{\ell j}}  
  \addConstraint{}{}{\dot{\underline{q}}_{\ell j} \leq p_i^\prime\theta_{\ell j}\dot{s}\leq \dot{\overline{q}}_{\ell j}}
  \addConstraint{}{}{\ddot{\underline{q}}_{\ell j} \leq p^{\prime\prime}_i\theta_{\ell j} \dot{s}^2 \leq \ddot{\overline{q}}_{\ell j}}
 \end{mini}

where $E(\theta_A,\theta_B)$ is the error on the Cartesian as defined in \eqref{eq:error} and $\epsilon$ a tolerance on it. In the previous expressions we have omitted that the constraints need to hold for all $i=\left\{0,\ldots, N\right\}$, $j\! =\! \left\{1,\ldots,n\right\}$, and $\ell :=A,B$.

The kinematic and dynamic constraints involving the products of different decision variables and the nonlinear relationship between joints and Cartesian paths (which makes the function $E(\theta_A,\theta_B)$ nonconvex) make problem \eqref{opt:generaldiscrete} a non-convex optimization problem. Note as well that depending on the choice of coordinate frame one writes the Cartesian coordinates in, {the function $E(\theta_A,\theta_B)$ will have cross products between decision variables $\theta_A$ and $\theta_B$.

Instead of relying on general non-convex solvers, e.g.,~\cite{byrd2000trust,waltz2006interior,gill2019practical}, we propose to exploit the bi-level structure that arises in this problem to develop efficient algorithms. In particular, inspired in \cite{vers2009}, \cite{roberts2015} and our previous work \cite{friedarxiv},
we obtain a structure in which \emph{given} the joint trajectories, the lower-level problem solves for $\dot{s}$, and the upper-level solves for the joint parameters $\theta_A,\theta_B$. 
This algorithm benefits from reduced running time as we establish empirically in Section \ref{sec_numerical}.

\section{A Bi-level Formulation for the Dual Arm Minimum Time Problem}\label{sec_form}

In this section, we propose a bi-level problem with a convex low-level subproblem, whose solution is the same as that to \eqref{opt:generaldiscrete}. 

\subsection{A Bi-level Formulation with Convex Low-Level Problem}\label{sec_bilevel}

We start by defining the lower-level problem
\begin{mini}|s|
  {\dot{s}^2\geq 0}{ \frac{1}{\dot{s}^2}}{\label{opt:inner}}{V(\theta_A,\theta_B)=}
  \addConstraint{-\dot{\underline{q}}_{\ell j}^2}{\leq \mbox{sign}(p_i^\prime\theta_{\ell j})(p_i^\prime\theta_{\ell j})^2\dot{s}^2 \leq \dot{\overline{q}}_{\ell j}^2}{}
\addConstraint{\ddot{\underline{q}}_{\ell j} }{\leq p^{\prime\prime}_{ij}\dot{s}^2 \leq \ddot{\overline{q}}_{\ell j}}{},
\end{mini}
%
As in \eqref{opt:generaldiscrete} the constraints need to be satisfied for all $j=1,\ldots,n$, for all $i=0,\ldots,N$, and for all robots $\ell :=A,B$. Note that $V(\theta_A,\theta_B)$ is the square of the optimal traverse time for a given trajectory defined by $\theta_A$ and $\theta_B$. 

Then, let us define the upper-level problem 
\begin{mini}|s|
  {\theta \in \mathbb{R}^{nd}}{ V(\theta_A,\theta_B)}{\label{opt:outer}}{({t_f}^\star)^2=}
  \addConstraint{E(\theta_A,\theta_B)}{ \leq \epsilon}{}
  \addConstraint{\underline{q}_{\ell j} }{\leq p_i\theta_{\ell j}\leq \overline{q}_{\ell j}}{}  \end{mini}
where in the previous expression the second constraint needs to hold for $i=0,\ldots, N$, $j=1,\ldots,n$, and $\ell:=A,B$. 

The structure proposed here follows the intuition discussed in Section \ref{sec_problem_formulation}, where the upper-level problem is independent of the kinematic and dynamic variables and the lower-level problem minimizes the traversal time given a trajectory to respect the manipulators' limits.

Two properties of the lower-level \eqref{opt:inner} make it appealing in this application. The problem has a closed form solution, and the function $V(\theta_A,\theta_B)$ is convex. This is the subject of the following Theorem. 
\begin{theorem}\label{thm:inner_cte}
Consider the optimization problem \eqref{opt:inner} parameterized by $\theta_A$ and $\theta_B$. If the velocity and acceleration joint limits satisfy $\dot{\underline{q}}_{\ell j}, \ddot{\underline{q}}_{\ell j} < 0$, $\dot{\overline{q}}_{\ell j}, \ddot{\overline{q}}_{\ell j} > 0$, for $\ell=A, B$ and $j=1,\ldots, n_\ell$, then $V(\theta_A,\theta_B)$, the value of the optimization problem \eqref{opt:inner}, is given by

\begin{dmath}\label{eq_solution_inner_cte}
    V(\theta_A,\theta_B) \!\!=\!\!\!\!\!\! \max_{\substack{\ell \in\left\{A,B\right\} \\
    j \in\{1,\dots,n_\ell\} \\ i\in\{0,\ldots,N\}}} \!\!\!\left\{\!\!\left(\!\!\max\left\{\frac{p'_i\theta_{\ell j}}{\dot{\overline{q}}_{\ell j}},\frac{p^{\prime}_i\theta_{\ell j}}{\dot{\underline{q}}_{\ell j}}\right\}\!\!\right)^2\!\!\!\!\!,  \max\!\!\left\{\frac{p''_i\theta_{\ell j}}{\ddot{\overline{q}}_{\ell j}},\frac{p^{\prime\prime}_i\theta_{\ell j}}{\ddot{\underline{q}}_{\ell j}}\right\}\!\!\!\right\}.
\end{dmath}

Furthermore, $V(\theta_A,\theta_B)$ is convex and $(\dot{s}^2)^\star=1/V(\theta_A,\theta_B)$.

\end{theorem}    
\begin{proof}
    Proof is analogous to the single arm in \cite{friedarxiv}.
\end{proof}

From the expression for $V(\theta_A,\theta_B)$ given in the Theorem and the definition of the upper-level problem in \eqref{opt:outer} it follows that to reduce the time of completion one needs to minimize the function $V(\theta_A,\theta_B)$. This can be done by following subgradient descent algorithms on the function since it is convex. Yet, we are required to maintain the error defined in \eqref{eq:error} below a given threshold $\epsilon$ as described by \eqref{opt:outer}. This is the subject of the following section.

\subsection{Coupled Manipulator Single Chain Kinematics}
\label{sec:dualarmsol}

As we are interested in the relative path between the two robots, we model both arms as a single kinematic chain, which, as discussed in Section \ref{sec_problem_formulation}, simplifies the computation of the path error~\eqref{eq:error} and its gradient. Without loss of generality, in this work we consider that robot B is holding the part that defines the desired relative path $\chi_d$ (see Figure \ref{fig:transforms}). The chain begins at Robot B’s TCP frame $\mathcal{F}_{tB}$, passes through its base $\mathcal{F}_{bB}$, then Robot A’s base $\mathcal{F}_{bA}$, and ends at Robot A’s TCP $\mathcal{F}_{tA}$. Let $\mathcal{F}_{tB}, \mathcal{F}_{bB}, \mathcal{F}_{bA}, \mathcal{F}_{tA}$ be defined as in Section~\ref{sec_problem_formulation}, and $T_{bB,bA}$ the constant transformation from $\mathcal{F}_{bB}$ to $\mathcal{F}_{bA}$. The pose of $\mathcal{F}_{tA}$ relative to $\mathcal{F}_{tB}$ is then:

\begin{equation}
    T_{tB,tA} = T^{-1}_{bB,tB}T_{bB,bA}T_{bA,tA},
    \label{eq:reltransform}
\end{equation}
where $T_{bB,tB},T_{bA,tA}$ can be obtained from the forward kinematics \eqref{eq:fwd_kin_i} of robots B and A, respectively. 

With position and orientation obtained from $T_{tB,tA}$, then the relative forward kinematics is equal to:
\begin{equation}
    \chi = k(q), \quad\quad \dot{\chi} = J(q)\dot{q},
    \label{eq:fwd_kin_rel}
\end{equation}
where $\chi \in \mathbb{R}^m$ is the cartesian pose of the TCP of A with respect to the TCP of B, and $q^T = \left[q_B^T,\;q_A^T\right] \in \mathbb{R}^n$ is the coupled joint vector. 

With these definitions, we are in conditions of providing the expression for the error exploiting the single kinematic chain defined in \eqref{eq:fwd_kin_rel}.

\begin{prop}

Let $\chi^d$ be the desired cartesian curve that the relative motion of the robots should follow. The error $E(\theta_A,\theta_B)$ in \eqref{eq:error} can be written as 
    \begin{equation}
    \label{eq:relativeerror}
          E(\theta_A,\theta_B) = \left(\sum_{i=0}^N \left\|(\chi^d_i)_{\mathcal{F}_{tB}}- (k(\left(\!I_{n}\otimes p_i\!\right)\theta)_{\mathcal{F}_{tB}}\right\|_2^p\right)^{\frac{1}{p}},
    \end{equation}
    where $q_{i}\!=\!\left(\!I_{n}\otimes p_i\!\right)\theta$, $\otimes$ denotes the Kronecker product, $I_{n}$ is a $n \times n$ identity matrix and $\theta = \left[\theta_{B}^T,\theta_{A}^T\right]^T$.
\end{prop}
\begin{proof}
    Consider any point in $\chi^d_i$ in the desired relative path $\chi_d$. By definition, as robot B is holding the part that represents $\chi_d$, the transformation $T_{q_{n_B},\chi^d_i}$ is a transformation between the last joint of the manipulator and the tool frame. This transformation does not depend on any joints of the manipulator and is always constant. It is possible to determine this transformation for any point $\chi^d_i$, thus $(\chi_d)_{\mathcal{F}_{tB}}$ is agnostic to the joint positions of robots A and B.

    Consider the relative Cartesian path given in \eqref{eq:relcartesianpath}, in the reference frame of TCP B $\mathcal{F}_{tB}$. 
     \begin{equation}
    (\chi)_{\mathcal{F}_{tB}} = (\chi_A)_{\mathcal{F}_{tB}} - (\chi_B)_{\mathcal{F}_{tB}}.
\end{equation}
It follows that $(\chi_B)_{\mathcal{F}_{tB}} = 0$ as it describes the movement of the frame itself. Furthermore, $(\chi_A)_{\mathcal{F}_{tB}}$ can be obtained from equations \eqref{eq:reltransform} and \eqref{eq:fwd_kin_rel}, so that
     \begin{equation}
    (\chi)_{\mathcal{F}_{tB}} = k(q) = k(\left(\!I_{n}\otimes p_i\!\right)\theta)
\end{equation}
with the frame $()_{\mathcal{F}_{tB}}$ omitted for the remainder of the paper for brevity. This concludes the proof.   
\end{proof}

In practice, the curve $\chi_d $ in the frame of the tooltip can be defined by the desired curve in the single arm case (e.g., a CAD file that describes the part) shifted and rotated to the tooltip of robot B that will now grasp it, with a certain pick-up point and rotation - naturally defining the homogeneous transformations $T_{tb,\chi_i^d}$.

By choosing the frame $\mathcal{F}_{tB}$ as the frame of reference, the desired cartesian path $\chi_d$ is independent of the robot joint positions $q_A$ and $q_B$. This choice simplifies the Cartesian error constraint \eqref{eq:error}. Additionally, it reduces the number of computations per iteration, as we do not need to recalculate the desired cartesian path when robot poses change.

Next, to be able to compute the gradient of the Cartesian error, we will need an explicit expression of the Jacobian matrix. This is the subject of the next proposition.

\begin{prop}\label{prop_jacobian}
Consider the relative Cartesian coordinates of the single-loop kinematic chain given in \eqref{eq:fwd_kin_rel}. Then, the relative Jacobian $J(q)$ can be calculated as 
\begin{equation}
    J(q) = [-\psi_{tB,tA}\Omega_{tB,bB}J_B\quad \Omega_{tB,bA}J_A],
    \label{eq:reljacob}
\end{equation}
where
\begin{equation}
    \psi_{tB,tA} = \left[\begin{array}{cc}
        I & S(x_{tB,tA}) \\
         0 & I
    \end{array}\right], \qquad \Omega_{i,j} = \left[\begin{array}{cc}
       R_{i,j}  & 0 \\
       0  & R_{i,j}
    \end{array}\right],
\end{equation}
$S()$ is the skew-symmetric operator, $x_{i,j} \in \mathbb{R}^3$ and $R_{i,j} \in \mbox{SO3}$ are the distance and rotation between frames of reference $i$ and $j$, respectively, which in turn depend on the joint positions of the dual arm robot system. 
\end{prop}
\begin{proof}
    See \cite{roberts2015}. 
\end{proof}
With this formulation it is possible to individually check the contribution of each manipulator to the Cartesian velocity. Notably, this form of the Jacobian holds due to the particular choice of frame of reference. This expression is key to determining directions in which the joint parameters $\theta_A,\theta_B$ can be updated without increasing the error. Additionally, the Jacobian of each individual manipulator is generally provided by its manufacturer, facilitating obtaining its expression.

The next result provides the expression for the gradient of $E(\theta_A,\theta_B)$ in \eqref{eq:relativeerror}.

\begin{coro}\label{coro_gradient} 
Let $E(\theta_A,\theta_B)$ be the error for the single chain forward kinematics defined in \eqref{eq:relativeerror}. Then,  
\begin{equation}\label{eqn_error}
\nabla_\theta E(\theta) = 
\sum_{i=0}^N \left(\frac{\partial E(\theta)}{\partial \chi_i} \right)  J(q_i)(I_n\otimes p_i),
\end{equation}
where $J(q_i)$ denotes the Jacobian of the single chain forward kinematics \eqref{eq:reljacob} for the joint at the $i$-th point of the path $q_i=I_n\otimes p_i\theta$ and $\partial E/\partial \chi_i$ denotes the derivative with respect to the position and orientation of manipulator A's tooltip at the $i$-th point in the path.
\end{coro}
\begin{proof}
    The proof follows from applying the chain rule and Proposition \ref{prop_jacobian}.
\end{proof}
Considering that the manipulator is redundant, the Jacobian $J(q)$ must have a nullspace~\cite[Ch. 10]{siciliano2007}, thus if $(I_n\otimes p_i)\Delta\theta$ is in the nullspace of $J(q_i)$ it is feasible to change the pose of the manipulator without changing the path followed by the TCP. With this intuition, Corollary \ref{coro_gradient} exploits the redundancy of the dual-arm to reduce the traversal time without increasing the Cartesian error.

With this formulation, Problem \eqref{opt:outer} can be solved in similar form to \cite{friedarxiv}. In particular, we consider a primal-dual type of method (\cite[Ch. 11]{boyd2004convex}), to find a local minima, due to the simplicity of its implementation. Let $\lambda, \mu,\nu $ be in the positive orthants in $\mathbb{R},\;\mathbb{R}^{n(N+1)},\;\mbox{and}\;\mathbb{R}^{n(N+1)}$, respectively.  Define the Lagrangian for the outer subproblem \eqref{opt:outer} as
\begin{multline}   
    \label{eq:Lagrangian}
    \mathcal{L}(\theta,\lambda,\mu,\nu) = V(\theta) + \lambda \left(E(\theta) - \epsilon\right) \\ + \sum_{i=0}^N\sum_{j=1}^n \mu_{ij}(p_i\theta_j - \overline{q}_j)  -  \sum_{i=0}^N\sum_{j=1}^n \nu_{ij}(p_i\theta_j - \underline{q}_j).
\end{multline}

Then, $\theta$ is updated following subgradient descent with step-size $\eta_{\theta}$ and the dual variables are updated with gradient descent, with step-size $\eta_{d}$.

In the next section, we show numerical results that illustrate the efficiency of this dual arm formulation, and show the benefits it obtains when compared to the single arm formulation previously presented in \cite{fried2024}.

\section{Numerical Simulations}\label{sec_numerical}

This section presents simulation results to illustrate the benefits of the proposed approach to optimizing joint trajectory for dual redundant manipulators under constant path speed. In particular, we establish the benefits\textemdash in terms of traversal time, Cartesian error, and algorithm running time \textemdash compared to the single-arm case we presented in \cite{fried2024} and to CasADi \cite{Andersson2018} as a choice of general non-linear solver.

The simulation models a material deposition task using two planar 3R manipulators (see Figure \ref{fig:setup}), where one of the robots holds the spray, while the other the workpiece. The curve $\chi^d$ to be followed represents the leading edge of a 2D fan blade mock-up, discretized into $N+1 = 500$ uniformly spaced points. A Cartesian error tolerance of $\pm5$ mm is set per industry standards.

The length of each manipulator's links are $a_1 = 2m$, $a_2 = 1.5m$, $a_3 = 1m$ and their joint limits are given by $\dot{\overline{q}} = -\dot{\underline{q}} =[1.75\;1.57\;1]\;\;rad/s$ and $\ddot{\overline{q}} = -\ddot{\underline{q}} = [35\;31.4\;20]\;\;rad/s^2$.
 These values are based on the FANUC M-1000iA industrial robot to handle the reach and size of the fanblade.\footnote{The typical industrial robot of this size would usually have at least 6 degrees of freedom we have reduced it to a 3R manipulator for the purposes of this numerical demonstration.}. The simulation setup can be seen in Figure \ref{fig:setup}.

 \begin{figure}
     \centering
     \includegraphics[width=9cm]{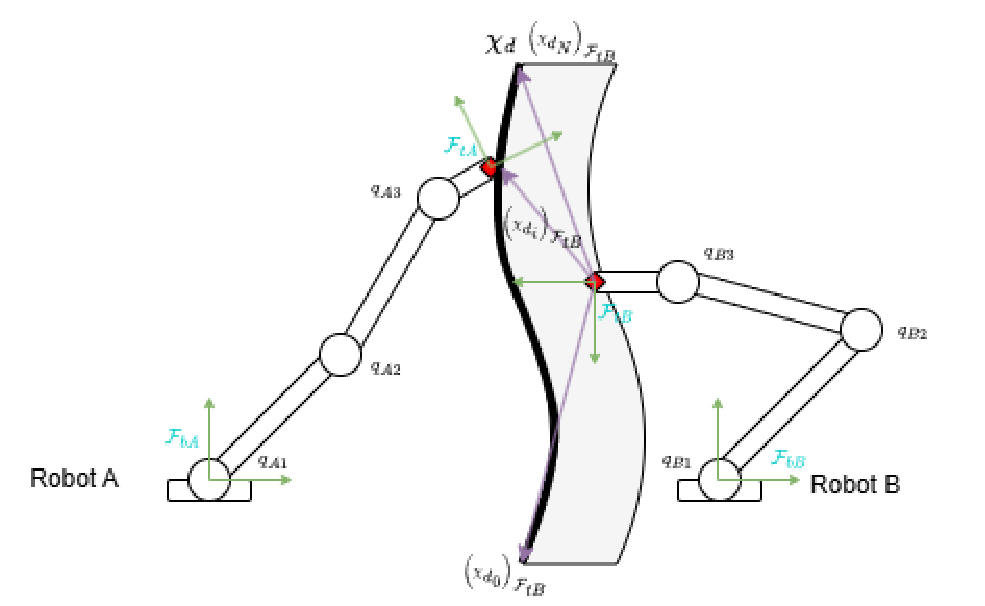}
     \caption{Dual Arm 3R planar setup. Robot B is holding the part while Robot A traces the curve with its TCP. Both robots move while coordinating to speed up the process.}
     \label{fig:setup}
 \end{figure}


    The forward kinematics mapping $k_\ell(q)$ for a 3R planar robot manipulator is given as follows
    \begin{equation}\label{eqn_forward_kinematics_experiment}
        \left[\begin{array}{c}
             x  \\
             y  \\
             \phi
        \end{array}\right] = \left[\begin{array}{c}
             a_1c_1 + a_2c_{12}+a_3c_{123}  \\
             a_1s_{1} + a_2s_{12}+a_3s_{123}  \\
             q_{\ell_1} + q_{\ell_2} + q_{\ell_3}
        \end{array}\right],
    \end{equation}
    where $c_{ijk} = \cos(q_{\ell_i} + q_{\ell_j} + q_{\ell_k})$, $s_{ijk} = \sin(q_{\ell_i} + q_{\ell_j} + q_{\ell_k})$.

In the single-arm case, as only the curve position is considered, $\phi$ is the free variable, and the key constraints are the first two rows of $k(q)$ and $J(q)$. $\phi$ is initialized uniformly at random in $[0, 2\pi]$ and held constant along the trajectory. For five experiments, initial $\phi$ values are $\left[5.12 \; 0.00 \; 1.04 \; 0.56 \; 4.76\right]$. The second arm initially holds the piece at the center of path $s$ with $\phi = \pi$, constant. For each $\phi$, initial $q_A, q_B$ are computed via inverse kinematics $k_\ell^{-1}$ according to \eqref{eqn_forward_kinematics_experiment} and parameterized over $s$ using a 9th-degree polynomial basis. The relative base position $p_{bA,bB} = [4, 0]^T$ and rotation $R_{bA,bB} = I$ remain constant. In the single-arm case, only robot A is optimized while robot B stays fixed. Both single and dual-arm cases share the same initial $t_f$ and Cartesian error due to identical initial conditions.



%
\begin{figure}[]
            \centering
            \includegraphics[width=5.4cm]{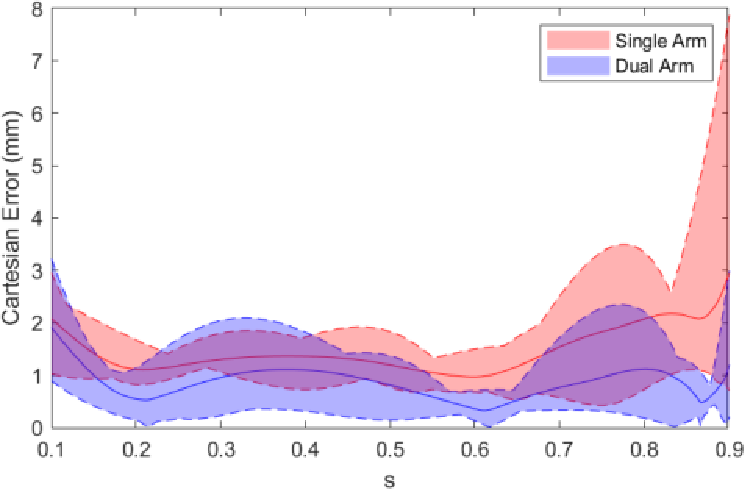}
            \caption{Cartesian errors for the whole trajectory, for single and dual arm solvers.  The improvement given by the extra redundancy is noticeable.}
	      \label{fig:sim1}
            
\end{figure}

\begin{table*}[h]
\centering
\begin{tabular}{|cccccccccc|}
\hline
\multicolumn{10}{|c|}{\textbf{Simulation Results}}                                                                                                                                                                                                                                                                                                                                                                                                                                                                                                                                                                                          \\ \hline
\multicolumn{1}{|c|}{}    & \multicolumn{3}{c|}{\textbf{Single Arm}}                                                                                                                                                                   & \multicolumn{3}{c|}{\textbf{Dual Arm - Proposed}}                                                                                                                                                          & \multicolumn{3}{c|}{\textbf{Dual Arm - CasADi}}                                                                                                                                       \\ \hline
\multicolumn{1}{|c|}{Run} & \multicolumn{1}{c|}{$t_f(s)$} & \multicolumn{1}{c|}{\begin{tabular}[c]{@{}c@{}}Running\\ Time (s)\end{tabular}} & \multicolumn{1}{c|}{\begin{tabular}[c]{@{}c@{}}Max Cartesian \\ Error (mm)\end{tabular}} & \multicolumn{1}{c|}{$t_f(s)$} & \multicolumn{1}{c|}{\begin{tabular}[c]{@{}c@{}}Running\\ Time (s)\end{tabular}} & \multicolumn{1}{c|}{\begin{tabular}[c]{@{}c@{}}Max Cartesian \\ Error (mm)\end{tabular}} & \multicolumn{1}{c|}{$t_f(s)$} & \multicolumn{1}{c|}{\begin{tabular}[c]{@{}c@{}}Running\\ Time (s)\end{tabular}} & \begin{tabular}[c]{@{}c@{}}Max Cartesian \\ Error (mm)\end{tabular} \\ \hline
\multicolumn{1}{|c|}{1}   & \multicolumn{1}{c|}{0.961}    & \multicolumn{1}{c|}{724}                                                        & \multicolumn{1}{c|}{7.9}                                                                 & \multicolumn{1}{c|}{0.821}    & \multicolumn{1}{c|}{1325}                                                       & \multicolumn{1}{c|}{4.9}                                                                 & \multicolumn{1}{c|}{86.2}     & \multicolumn{1}{c|}{4767}                                                       & 1485                                                                \\ \hline
\multicolumn{1}{|c|}{2}   & \multicolumn{1}{c|}{0.771}    & \multicolumn{1}{c|}{655}                                                        & \multicolumn{1}{c|}{10.1}                                                                & \multicolumn{1}{c|}{0.681}    & \multicolumn{1}{c|}{1319}                                                       & \multicolumn{1}{c|}{2.9}                                                                 & \multicolumn{1}{c|}{0.272}    & \multicolumn{1}{c|}{3858}                                                       & 1526                                                                \\ \hline
\multicolumn{1}{|c|}{3}   & \multicolumn{1}{c|}{0.837}    & \multicolumn{1}{c|}{767}                                                        & \multicolumn{1}{c|}{3.0}                                                                 & \multicolumn{1}{c|}{0.817}    & \multicolumn{1}{c|}{1628}                                                       & \multicolumn{1}{c|}{3.3}                                                                 & \multicolumn{1}{c|}{285}      & \multicolumn{1}{c|}{4656}                                                       & 1518                                                                \\ \hline
\multicolumn{1}{|c|}{4}   & \multicolumn{1}{c|}{0.804}    & \multicolumn{1}{c|}{838}                                                        & \multicolumn{1}{c|}{2.6}                                                                 & \multicolumn{1}{c|}{0.632}    & \multicolumn{1}{c|}{4824}                                                       & \multicolumn{1}{c|}{1.3}                                                                 & \multicolumn{1}{c|}{0.759}    & \multicolumn{1}{c|}{4713}                                                       & 1645                                                                \\ \hline
\multicolumn{1}{|c|}{5}   & \multicolumn{1}{c|}{0.952}    & \multicolumn{1}{c|}{797}                                                        & \multicolumn{1}{c|}{2.7}                                                                 & \multicolumn{1}{c|}{0.664}    & \multicolumn{1}{c|}{3290}                                                       & \multicolumn{1}{c|}{3.6}                                                                 & \multicolumn{1}{c|}{11.8}     & \multicolumn{1}{c|}{2919}                                                       & 420                                                                 \\ \hline
\multicolumn{1}{|c|}{Avg} & \multicolumn{1}{c|}{0.865}    & \multicolumn{1}{c|}{756}                                                        & \multicolumn{1}{c|}{5.3}                                                                 & \multicolumn{1}{c|}{0.723}    & \multicolumn{1}{c|}{2477}                                                       & \multicolumn{1}{c|}{3.3}                                                                 & \multicolumn{1}{c|}{76.95}    & \multicolumn{1}{c|}{4182}                                                       & 1318                                                                \\ \hline
\multicolumn{1}{|c|}{Std} & \multicolumn{1}{c|}{0.077}    & \multicolumn{1}{c|}{63}                                                         & \multicolumn{1}{c|}{3.1}                                                                 & \multicolumn{1}{c|}{0.080}    & \multicolumn{1}{c|}{1382}                                                       & \multicolumn{1}{c|}{1.4}                                                                 & \multicolumn{1}{c|}{109}      & \multicolumn{1}{c|}{714}                                                        & 452                                                                 \\ \hline
\end{tabular}
\caption{Comparison of single- and dual-arm methods using the proposed solver and CasADi across five initial conditions with identical hyperparameters. Table reports final time, runtime, and max Cartesian error for each method.
}

\end{table*}

%

 Simulations were run on a Dell Precision 3650 Tower with an 11th Gen Intel i9-11900K, 128GB RAM, Windows 11 Pro, and Matlab R2021b. All experiments used 2000 outer-loop iterations of a primal-dual algorithm with $\eta_{\theta}=$ $10^-5$, $\eta_{d}=$ $0.5$, and path constraint tolerance $\epsilon = 10^-5$. We used the Matlab version of CasADi with IPOPT as the solver, limited to 3000 iterations.

            Optimization results in Table~1 show a clear performance gap between the single- and dual-arm cases. The solutions are local, with final times varying by initial condition. Average path velocities are $\dot{s} = 1.16 \pm 0.11$ (single arm) and $1.40 \pm 0.16$ (dual arm).

            The dual-arm approach also outperforms in Cartesian error, leveraging the additional redundancy via $\partial L/\partial \theta$, given by the larger nullspace of the relative Jacobian $J(q)$ compared to $J(q_A)$. From Table 1, the average max Cartesian error for the dual arm case is $3.3 \pm 1.4 mm$, which is below the $5mm$ stipulation, whereas the maximum Cartesian error of the other scenario is $5.3 \pm 3.1mm$, above the tolerance.  Figure \ref{fig:sim1} makes this difference clear, showing that the average Cartesian error of the dual arm scenario was below the average error of the single arm case for every point $s_i$ in the path. In general, the extra degrees of freedom allow the dual arm to obtain reduced times when the single arm case already starts to struggle to keep the error within bounds, as evidenced by the tail end of the curve.

           In contrast, CasADi failed to converge under the same initial conditions, despite taking nearly 60\% extra algorithm running time when compared to the proposed solver. Furthermore, it also failed when the inner problem was pre-solved, due to the non-differentiability of $V(\theta)$ in~\eqref{eq_solution_inner_cte}.

Overall, these results confirm the advantages of the proposed bi-level dual-arm optimization approach.


\section{Conclusions}\label{sec_conclusion}

This work addresses joint motion optimization for a redundant dual-arm robot tracking a relative Cartesian path at constant speed. We efficiently compute local solutions by reformulating it as a bi-level problem, with a convex, closed-form low-level subproblem that is independent of both the path and manipulators. By defining Cartesian error in the TCP frame, we model the system as a single kinematic chain, decoupling the desired path from joint configurations. Using the low-level solution, we compute subgradients for the high-level problem. The approach was validated in a cold spraying-inspired task and compared to a single-arm baseline and CasADi.




\bibliographystyle{IEEEtran}
\bibliography{cdc2023bib}             

\end{document}